\documentclass[runningheads]{llncs}

\usepackage{eccv}

\usepackage{eccvabbrv}

\usepackage{graphicx}
\usepackage{booktabs}

\usepackage[accsupp]{axessibility}  % Improves PDF readability for those with disabilities.

\usepackage{hyperref}

\usepackage{orcidlink}

\usepackage{url}
\usepackage{graphicx}
\usepackage{multirow}
\usepackage{mathtools}
\usepackage{enumitem}
\usepackage{subcaption} 
\usepackage{bm}
\usepackage{booktabs}

\usepackage{enumitem}
\usepackage{pifont}% http://ctan.org/pkg/pifont
\usepackage{mathtools}
\usepackage{xspace}
\usepackage{bm}
\usepackage{subcaption}
\DeclareRobustCommand\onedot{\futurelet\@let@token\@onedot}
\def\@onedot{\ifx\@let@token.\else.\null\fi\xspace}
\makeatletter
\DeclareRobustCommand\onedot{\futurelet\@let@token\@onedot}
\def\@onedot{\ifx\@let@token.\else.\null\fi\xspace}

\def\eg{\emph{e.g}\onedot} 
\def\ie{\emph{i.e}\onedot}

\def\etal{\emph{et al}\onedot}

\usepackage{booktabs}% table utilities
\usepackage{siunitx}% number and symbol alignment

\usepackage{thmtools, thm-restate}
\begin{document}

% ---------------------------------------------------------------
% TODO REVIEW: Replace with your title
\title{From Fake to Real: Pretraining on Balanced Synthetic Images to Prevent Spurious Correlations in Image Recognition} 

% TODO REVIEW: If the paper title is too long for the running head, you can set
% an abbreviated paper title here. If not, comment out.
\titlerunning{FFR: Pretraining on Balanced Synthetic Images to Prevent Bias}

% TODO FINAL: Replace with your author list. 
% Include the authors' OCRID for the camera-ready version, if at all possible.
\author{Maan Qraitem\inst{1}  \and
Kate Saenko\inst{1}  \and
Bryan A. Plummer \inst{1}}

% TODO FINAL: Replace with an abbreviated list of authors.
\authorrunning{M.~Qraitem et al.}
% First names are abbreviated in the running head.
% If there are more than two authors, 'et al.' is used.

% TODO FINAL: Replace with your institution list.
\institute{$^1\text{Boston University}$\\
\{\tt\small mqraitem,saenko,bplum\}@bu.edu\\}

\maketitle

\begin{abstract}
    Visual recognition models are prone to learning spurious correlations induced by a biased training set where certain conditions $B$ (\eg, Indoors) are over-represented in certain classes $Y$ (\eg, Big Dogs). Synthetic data from off-the-shelf large-scale generative models offers a promising direction to mitigate this issue by augmenting underrepresented subgroups in the real dataset. However, by using a mixed distribution of real and synthetic data, we introduce another source of bias due to distributional differences between synthetic and real data (\eg synthetic artifacts). As we will show, prior work's approach for using synthetic data to resolve the model's bias toward $B$ do not correct the model's bias toward the pair $(B, G)$, where $G$ denotes whether the sample is real or synthetic. Thus, the model could simply learn signals based on the pair $(B, G)$ (\eg, Synthetic Indoors) to make predictions about $Y$ (\eg, Big Dogs). To address this issue, we propose a simple, easy-to-implement, two-step training pipeline that we call From Fake to Real (FFR). The first step of FFR pre-trains a model on balanced synthetic data to learn robust representations across subgroups. In the second step, FFR fine-tunes the model on real data using ERM or common loss-based bias mitigation methods. By training on real and synthetic data separately, FFR does not expose the model to the statistical differences between real and synthetic data and thus avoids the issue of bias toward the pair $(B, G)$. Our experiments show that FFR improves worst group accuracy over the state-of-the-art by up to 20\% over three datasets. Code available: \url{https://github.com/mqraitem/From-Fake-to-Real} 
    \end{abstract}

\section{Introduction}

\begin{figure}[t]
\centering
\includegraphics[width=\linewidth]{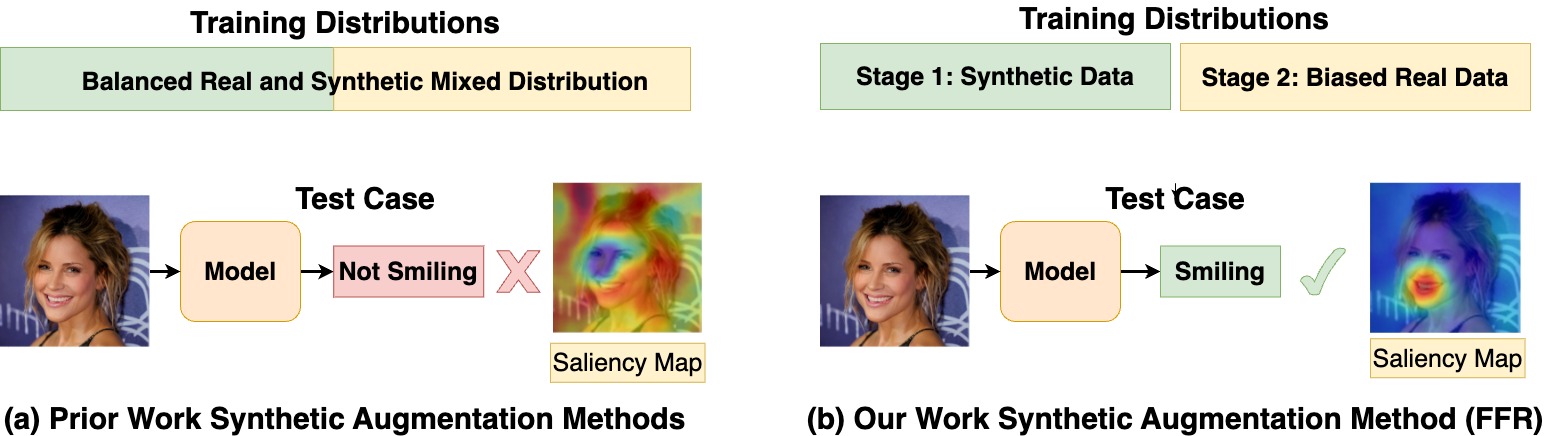}
% \vspace{-6mm}
\caption{Comparison of RISE \cite{Petsiuk2018rise} saliency maps produced for a model trained to predict the attribute Smiling $Y$ given a bias toward Gender $B$ (Most Women are not Smiling) using: (a)  prior work in Synthetic Augmentation  \cite{Ramaswamy_2021_CVPR, Mondal2023MinorityOF}, which do not address the unexpected bias toward the pair (Gender $B$, Data Source $G$), \eg, (Female-Synthetic vs Female-Real) and, thus, use spurious features leading to an incorrect prediction (not smiling) (b)  our approach FFR where the Synthetic and Real Data are separated into two training stages, thereby mitigating the bias suffered by prior work enabling us to learn the correct features (Mouth) and correctly predict Smiling.} 
\vspace{-6mm}
\label{fig:figure_1}
\end{figure}

Visual recognition models are prone to learning spurious correlations (Bias) \cite{revise_olga,zhao2021captionbias,Meister2022GenderAI}. These correlations frequently arise due to an imbalance in the training set. For example, given a dataset with classes $Y$ (\eg, Smiling vs Not Smiling), there exists a confounding bias variable $B$ (Gender: Male and Female) in the training set such that one bias group (\eg, Male) is represented in one class more than others (\eg, most males are Smiling). This leads models to mistakenly use the bias signal $B$ (gender) to predict $Y$ (Smiling). Rapid progress in large-scale generative models, notably diffusion-based models \cite{Ho2020DenoisingDP, saharia2022photorealistic}, provides a clear mitigation method that alleviates bias using synthetic data: use a mixed distribution of synthetic and real data that alleviates the real dataset bias.  

Prior work has introduced several methods to achieve this goal. For example, Additive Synthetic balancing (ASB) \cite{Ramaswamy_2021_CVPR,Sharmanska2020contrastive} augments the biased real dataset with a balanced synthetic dataset. Uniform Synthetic Balancing (USB) generates enough data to uniformly balance the dataset subgroups \cite{wang2020deep,Mondal2023MinorityOF}, \ie, each subgroup will have the same number of samples\footnote{Refer to the supplementary for a visual comparison of prior work in synthetic augmentation (\eg,~\cite{wang2020deep,Ramaswamy_2021_CVPR,Sharmanska2020contrastive,Mondal2023MinorityOF}) and our approach.}. However, by training the real and synthetic data samples at the same time, a model may simply learn to identify correlations between bias $B$ and whether the data was real or generated $G$ (\eg by using generative model artifacts \cite{Corvi_2023_ICASSP}). For example, in the setting where the training data contained mostly smiling men but few smiling women, prior work may simply learn that synthetically generated women smile (but women in real images do not). Thus, as shown in   Figure \ref{fig:figure_1}(a), models trained using strategies of prior work (\eg, ASB and USB) may focus on unrelated features for the target task. In addition, assuming some distributional differences between synthetic and real data, we provide a theoretical analysis showing that, in the standard
spurious-correlation setting considered in this work, every possible
augmentation of a biased dataset with synthetic data will exhibit some
bias toward $(B,G)$; i.e.,
$P_{\bar{D}}(Y \mid B,G) \neq P_{\bar{D}}(Y)$.

To mitigate this problem, we rethink how synthetic data is used for bias mitigation by developing a simple, easy-to-implement, yet effective two-stage training pipeline called From Fake to Real (FFR). The first stage involves pre-training on balanced synthetic data where we learn robust representations across subgroups. In the second step, FFR fine-tunes the model on real data using ERM or common loss based bias mitigation methods \cite{hong2021unbiased,8953715,ryu2018inclusivefacenet,Tartaglione_2021_CVPR,Sagawa2020DistributionallyRN}. By separating the two data sources (\ie Real and Synthetic) into two different training stages, FFR doesn't expose the model to the statistical differences between real and synthetic data (\eg generative model artifacts \cite{Corvi_2023_ICASSP}) and, thus, avoids the issue of bias that might arise from training on these two sources of data together. Effectively, the synthetic data acts as a source of unbiased representations for each subgroup, leading to improved performance when training with the real data using ERM or loss-based bias mitigation methods in the second step. As shown in Figure \ref{fig:figure_1}(b), this enables FFR to learn more relevant features rather than focusing on spurious background features.

To evaluate our approach, we expand on the experimental frameworks of prior work, which are limited to one bias rate per dataset  \cite{sagawa2019distributionally, qraitem2023bias, joshi2023towards}. Instead, we conduct systemic analysis over three datasets, CelebA-HQ \cite{CelebAMask-HQ}, UTK-Face \cite{zhifei2017cvpr}, and SpuCo Animals  \cite{joshi2023towards}, and a range of bias rates ranging from moderate to severe resulting in over 5k experiments in total.

Our contributions are summarized below:

\begin{itemize}[nosep,leftmargin=*]
    \item We introduce a simple, easy to implement, yet effective, two-step training pipeline (FFR) that uses synthetic data to alleviate the issue of spurious correlations (Bias). Unlike prior work, our pipeline avoids the issue of bias to distributional differences between real-synthetic data (\eg, generative model artifacts) and, thus, is more effective at mitigating bias.
    \item We provide a theoretical analysis on how augmentation with synthetic data results in an unexpected bias toward synthetic artifacts. 
    \item Comprehensive experiments over three datasets (UTK-face, CelebA-HQ, and SpuCo Animals) and at least four bias strengths per dataset validate our method's effectiveness. Indeed, FFR improves performance over state-of-the-art worst accuracy by up to $20\%$.  
\end{itemize}

\section{Related Work}

% \smallskip
\noindent\textbf{Mitigating Bias with Synthetic Data.} As noted in the Introduction, some limited work exists on using synthetic data augmentation to address issues due to imbalanced training data. This includes Uniform Synthetic Balancing (USB) \cite{wang2020deep,Mondal2023MinorityOF}, which balances underrepresented subgroups, where subgroups are the intersection of classes $Y$ and bias groups $B$. This, in turn, effectively ensures that $Y$ is statistically independent from $B$, \ie, $P_{\bar{D}}(Y | B) = P_{\bar{D}}(Y)$  where $\bar{D}$ is the combined dataset of real and synthetic data.  Additive Synthetic Balancing (ASB) \cite{Ramaswamy_2021_CVPR,Sharmanska2020contrastive} augments a biased real dataset with a balanced synthetic dataset. In our work, we show how both approaches (USB and ASB) result in models that are biased toward $(B, G)$ where $G = \{Real, Synthetic\}$, \ie, the variable that differentiates between real and synthetic data.   We could attempt to mitigate this issue by combining USB and ASB with loss-based bias mitigation methods (\eg, \cite{hong2021unbiased,8953715,ryu2018inclusivefacenet,Tartaglione_2021_CVPR,Sagawa2020DistributionallyRN}). However, in order to account for the new source of bias from $(B, G)$ where $G = \{Real, Synthetic\}$, this approach doubles the number of bias groups  ($|(B, G)| = |B||G| = 2|B|$) which increases the optimization difficulty, reducing performance as we will show in Section \ref{sec:systemic_analysis_exp}.  Instead, our two-stage training pipeline addresses the issue of new biases being introduced from using synthetic data by training both real and generated data separately.
\smallskip

\noindent\textbf{Synthetic-Data-Free Mitigation Methods.} Also related to our task are methods that use architecture changes and/or alter the training procedures to mitigate dataset bias \cite{ryu2018inclusivefacenet,8953715,Wang_2020_CVPR,hong2021unbiased,Tartaglione_2021_CVPR,sagawa2019distributionally}. For example, Sagawa~\etal~\cite{sagawa2019distributionally} proposes GroupDRO (Distributionally Robust Neural Networks for Group Shifts), a regularization procedure that adapts the model optimization according to the worst-performing group. More recently, a series of works seeks to mitigate bias assuming no access to bias labels in training time \cite{liu2021just,ahmed2020systematic,zhang2022correct,kirichenko2022last}. Most recently, DFR \cite{kirichenko2022last} showed how fine-tuning a model on a small balanced validation (after being trained on the biased training set) achieves state-of-the-art performance. Our work complements these efforts by introducing a novel pipeline for using synthetic data that further boosts the performance of these methods, especially in high-bias settings.  Thus, as we will show, these methods and our approach can be combined to boost performance over either when they are used alone.
\smallskip

\noindent\textbf{Uncovering Spurious Correlations.} In our work, we are interested in mitigating spurious correlations; a spurious correlation results from underrepresenting a certain group of samples (\eg, samples with the color red) within a certain class (\eg, planes) in the training set. This leads the model to incorrectly correlate the class with the over-represented group. For example, prior work  has shown that several datasets exhibit spurious correlations \cite{Meister2022GenderAI,li2023whac,hirota2022gender}. For example, Meister~\etal~\cite{Meister2022GenderAI} reports that models trained on COCO \cite{lin2014microsoft} and OpenImages \cite{kuznetsova2020open} learn spurious correlations with respect to various gender artifacts. Li~\etal~\cite{li2023whac} showed that models trained on ImageNet spuriously correlate the Carton class with Chinese watermarks. Hirota~\etal~\cite{hirota2022gender} showed how several VQA dataset encodes racial and gender biases.  Our work complements these effort by introducing a more effective way of using synthetic data to mitigate spurious correlations.

\section{Synthetic Data for Robust Representations against Bias in Image Recognition}

Visual classification models can often rely on spurious correlations in the training set that do not reflect their real-world distribution. More concretely, given a dataset of images $X$, classes $Y$, and bias signal $B$ (\eg, Gender: Male/Female), a biased model relies on the signal in $X$ that infer $B$ to make predictions $\hat{Y}$. This is because the distribution $P_D(Y | B) \neq P_D(Y)$, \ie, the training set encodes some correlation between the classes and the biases. For example, given a class $y$ (\eg, Smiling), a certain bias group $b$ (\eg, Male) might be over-represented compared to others. Therefore, a model might mistakenly predict the class of an image (\eg, Not Smiling) as the wrong class (\eg, Smiling) because the signal $b$ (Male) is present in the image (\eg, Man is Not Smiling) \cite{Sagawa2020DistributionallyRN,qraitem2023bias}.

To address this issue, our work explores using synthetic data from generative models as we will discuss in detail below. Section \ref{sec:motivation} explores how augmenting the real dataset with synthetic data results in a bias towards distributional differences between synthetic and real data. Section \ref{sec:ffr} introduces From Fake to Real (FFR); our novel two-stage pipeline that addresses this issue. 

\subsection{Motivation}
\label{sec:motivation}

In this section, we explore a critical problem with the class of solutions that mitigates dataset bias by augmenting biased datasets with synthetic data, \eg, Additive Synthetic Balancing (ASB) \cite{Ramaswamy_2021_CVPR,Sharmanska2020contrastive} and Uniform Synthetic Balancing (USB) \cite{wang2020deep,Mondal2023MinorityOF}. These approaches don't consider the fact that the distribution of synthetic data is not the same as the distribution of real data. Indeed, while research on generative models has made significant progress in producing ever more realistic images, especially with the recent advent of diffusion models \cite{Ho2020DenoisingDP, saharia2022photorealistic}, there might still be some distributional differences between the real and synthetic data. For example, Corvi~\etal~\cite{Corvi_2023_ICASSP} demonstrates how state-of-the-art diffusion models leave fingerprints in the generated images that recognition models could use to differentiate between real and synthetic data. 

Assuming real and synthetic data are drawn from different distributions,
we consider the standard spurious-correlation setting in which different
bias groups are disproportionately represented across target classes, as
formalized below. We argue that it is impossible to guarantee that we can
create $\bar{D}$ where $Y$ is not biased toward the pair $(B,G)$. Formally:

\begin{restatable*}{theorem}{theoremone}
\label{thm:1}
Assume we are given a dataset $D$ with target labels $Y$ and biased group
labels $B$, where there exist $y,y' \in Y$ and $b,b' \in B$ such that
\[
P_D(B=b \mid Y=y) > P_D(B=b' \mid Y=y)
\]
and
\[
P_D(B=b' \mid Y=y') > P_D(B=b \mid Y=y').
\]
Let $\bar{\mathcal{D}}$ represent all possible versions of the dataset
augmented with synthetic data, where
$G \in \{\mathrm{Real},\mathrm{Synthetic}\}$ denotes the data source.
Then, for every $\bar{D} \in \bar{\mathcal{D}}$,
\[
P_{\bar{D}}(Y \mid B,G) \neq P_{\bar{D}}(Y).
\]
\end{restatable*}

Refer to the supplementary for a proof. As shown, this Theorem guarantees that it is impossible to create an augmented version of the dataset $D$, \ie, $\bar{D}$, without $\bar{D}$ exhibiting some bias toward $(B, G)$. Therefore, this implies that both methods from prior work, ASB \cite{Ramaswamy_2021_CVPR,Sharmanska2020contrastive} and USB \cite{wang2020deep,Mondal2023MinorityOF}, may rely on biased signals stemming from $(B, G)$ to make predictions.

To gain some intuition, consider the following illustrative example for Uniform Synthetic Balancing (USB): in an attempt to mitigate the dataset bias of class Landbirds being mostly on Land and Waterbirds being most Water, a significant number of synthetic samples of Landbirds on Water and Waterbirds on Land are added to the dataset. While this means that there is an equal number of Landbirds and Waterbirds on Land and on Water in the combined dataset, \ie, $P_{\bar{D}}(Y | B) = P_{\bar{D}}(Y)$, this also means that there are significantly more \textit{Synthetic} Landbirds on Water than there are \textit{Synthetic} Landbirds on Land. Assuming that the model could differentiate between real and synthetic images, then it is likely advantageous to learn the signal pair (Water, Synthetic) in order to predict the class Landbird while the signal (Water, Real) predicts the class Waterbirds.

\begin{figure}[t!]
    \centering
    \includegraphics[width=\linewidth]{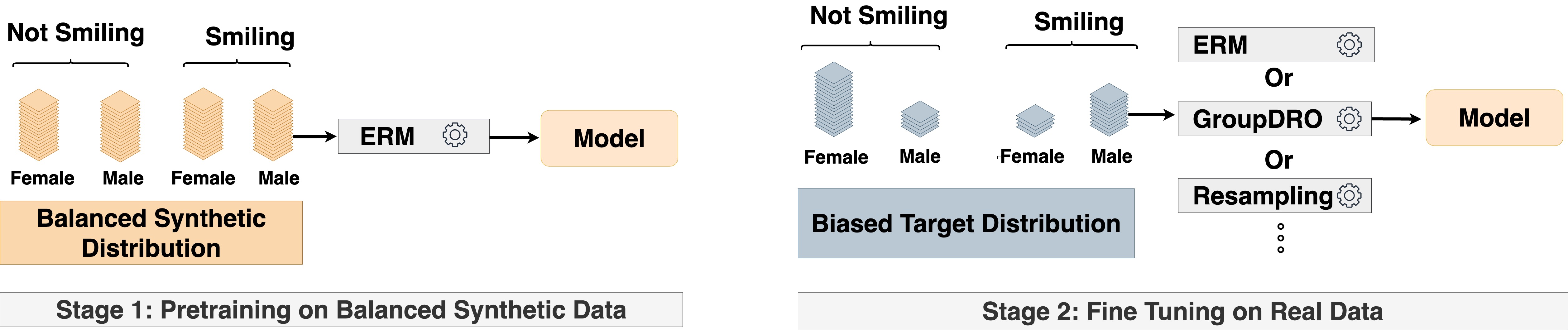}
    \caption{An overview of From Fake to Real (FFR) that incorporates synthetic data to mitigate bias. In Stage 1, we pretrain on a balanced synthetic dataset where we learn robust representations across subgroups. In Stage 2, we fine-tune the model on real data using ERM or common synthetic-data-free bias mitigation methods. By training on real and synthetic data separately, FFR does not expose the model to the statistical differences between real and synthetic data and thus avoids the issue of bias between the two data sources. Refer to Section \ref{sec:ffr} for further discussion. } 
    \label{fig:figure_2}
    \vspace{-6mm}
\end{figure}

\subsection{From Fake to Real (FFR): A Two-Stage Training Pipeline}
\label{sec:ffr}

Our approach, From Fake to Real (FFR), aims to address the issue in prior work outlined in Section \ref{sec:motivation}, where models learn a bias between the target labels $Y$ and the pair labels $(B, G)$. The key to our approach is the separation of training on the two data sources, real and synthetic, into two different stages. The model is exposed to one data source at a time, which effectively prevents the use of signals from the pair $(B, G)$ to make predictions as neither appear in the same training step. We provide additional details on our two training stages below: 
\smallskip

\noindent\textbf{Step 1:} FFR pretrains a model $M$ on a balanced synthetic dataset $D_{syn}$ where $P_{D_{syn}}(Y | B) = P_{D_{syn}}(Y)$. To obtain this distribution, we simply deploy a generative model to sample the same number of synthetic data per bias subgroup. This step enables the model $M$ to learn robust initial representations for each subgroup. Refer to Figure \ref{fig:figure_2} (Stage 1) for an overview of this step. Denote the resulting model from this step as $\bar{M}$. 
\smallskip

\noindent\textbf{Step 2:} While Step 1 learns valuable unbiased representations, there is still a distribution shift going from real to synthetic datasets \cite{sariyildiz2023fake}. Therefore, we fine-tune the model $\bar{M}$ from Step 1 on the real dataset to better fit to its distribution. We find that even a simple empirical-risk minimization fine-tuning using the model $\bar{M}$ as an initialization is sufficient to boost performance. However, the real dataset's distribution $D$ is biased, \ie, $P_D(Y \mid B) \neq P_D(Y)$. Thus, some of the benefits of our first stage pretraining are undone as the model might simply relearn the bias. To address this, we combine our two-stage training pipeline with loss-based bias mitigation methods (\eg, \cite{hong2021unbiased,8953715,ryu2018inclusivefacenet,Tartaglione_2021_CVPR,Sagawa2020DistributionallyRN}).  Refer to Figure \ref{fig:figure_2} (Stage 2) for an overview of this step. As we note in our experiments, regardless of the method used in Step 2, we observe a significant performance boost using Step 1's model $\bar{M}$ for initialization.
\smallskip

In summary, FFR is a flexible framework that rethinks the use of synthetic data for bias mitigation. We use FFR to deploy synthetic data to learn initial unbiased representations to improve the performance of training on real data regardless of the method used to train on real data. Therefore, it is generalizable to any bias mitigation method and easy to implement no matter the model architecture. Finally, our framework effectively avoids the issue of bias to distributional differences between real and synthetic data, unlike prior work's methods.

\section{Experiments}
\label{sec:main_exps}

\noindent\textbf{Datasets.} We seek to compare the effect of synthetic augmentation on varying amounts of bias. To that end, we use three standard bias mitigation datasets, namely: CelebA-HQ \cite{CelebAMask-HQ}, SpuCo Animals dataset \cite{joshi2023towards}, and  UTK-Face dataset \cite{zhifei2017cvpr}. SpuCo Animals has one possible bias variable (``Background'') where the bias is $95\%$ (\ie, the majority bias group takes up $95\%$ of the class distribution). Prior work that used CelebA \cite{sagawa2019distributionally} used the attribute ``Blonde Hair,''  which has $\sim 97\%$ bias, and ``Wearing Lipstick,'' which has $\sim 99.9\%$ bias. In contrast, prior work that used Utk-Face \cite{qraitem2023bias} used the Age attribute with $90\%$ bias. However, since we seek to study the effect of different bias settings on the methods' performance, simply comparing performance on different attributes with different biases is not fair,  as it entangles the difficulty of learning different targets (\eg, ``Wearing Lipstick'' vs. ``Blonde Hair'') and the difficulty of learning different bias ratios (97\% vs. 99.9\%). To mitigate this issue, we choose to fix the bias and target attribute per dataset and manually vary the bias according to 5 main ratios ranging from moderate to severe $(90\%, 95\%, 97\%, 99\%, 99.9\%)$ by simply dropping samples appropriately from the minority groups to match the target bias ratio. Specifically, we evaluate using: 1) CelebA-HQ \cite{CelebAMask-HQ}, where we choose ``Smiling'' as the target attribute and ``Gender'' as the bias attribute, 2) UTK-Face \cite{zhifei2017cvpr}, where we use ``Age'' as the bias attribute and ``Gender'' as the target attribute, and 3) SpuCo Animals \cite{joshi2023towards}, where the bias attributes are \{\textit{Indoors, Outdoors, Land, Water}\} and target attributes are \{\textit{Small dogs, Big Dogs, Landbirds, Waterbirds} \}. Note that SpuCo Animals has a minimum bias of $95\%$. Thus, $90\%$ bias is only used for UTK-Face and CelebA-HQ. 
\smallskip

\noindent\textbf{Metrics.} Following \cite{sagawa2019distributionally}, we use Worst Accuracy (WA) to measure the models' spurious behavior. This metric returns the accuracy of the worst-performing subgroup where the subgroup is defined as the intersection of class and bias groups. In addition, we use balanced accuracy (BA), which averages the accuracies of all subgroups \cite{qraitem2023bias}. BA reflects the overall performance of the model while not being biased by the majority of subgroups. 
\smallskip

\noindent\textbf{Implementation Details.} We use a Resnet50 \cite{He2016DeepRL} backbone trained with ADAM \cite{KB14Adam}, where we use grid search to set the learning rate over the validation set. We use default values for the other parameters. Furthermore, we find that freezing the batch norm in FFR Stage 2 to be helpful on some datasets. See the supplementary for additional implementation details for each method. For data generation, we use Stable Diffusion V1.4 \cite{Rombach_2022_CVPR}, where we use the prompt template \textit{A photo of \{\textbf{bias}\} \{\textbf{class}\}} to sample new images. As a powerful generator is not easily accessible for every application, in the supplementary we explore the effect of the quality of the synthetic images on performance.
\smallskip

\noindent\textbf{Methods.} We report the performance of training with three modes of incorporating synthetic data: 

\begin{itemize}[noitemsep,topsep=0pt]
    \item \textbf{None}: No synthetic data is used.
    \item \textbf{USB} \cite{Mondal2023MinorityOF}: Synthetic data is used to uniformly balance the distribution 
    \item \textbf{ASB} \cite{Ramaswamy_2021_CVPR}: A balanced synthetic dataset is added to the real dataset (ASB) 
    \item \textbf{FFR}: our method where we first pre-train on balanced synthetic data using ERM and then fine-tune on real data. 
\end{itemize}

For each mode, we report performance using Empirical Risk Minimization (ERM)  and several popular state-of-the-art bias mitigation methods. More concretely, we report the performance of GroupDRO \cite{sagawa2019distributionally}, Resampling, and Deep Feature Reweighting (DFR) \cite{kirichenko2022last}. GroupDRO is an optimization technique where the contribution of each subgroup loss is weighted by their performance. Resampling oversamples minority subgroups such that each subgroup is equally represented per batch. DFR fine-tunes a linear layer over the feature space on a balanced validation set. Note that Group-DRO and Resampling require access to Bias labels in the training set, while DFR does not.

Note that when no synthetic data is used, the bias mitigation methods are deployed to minimize the bias toward $B$. When combined with USB and ASB, the methods are deployed to mitigate the bias toward $(B, G)$. Finally, when the methods are combined with FFR, they are deployed to mitigate bias toward $B$ only since FFR minimizes the bias toward $G$ by definition. Finally, we fix the number of synthetic data samples used for each method: USB, ASB, and FFR. The fixed size is the number of samples required to balance the dataset in USB.

\subsection{Comparing Synthetic Augmentation Methods with ERM}
\label{sec:systemic_analysis_exp}

Figure \ref{fig:figure_3} compares the performance of our Synthetic Data Augmentation method (FFR) to prior work methods (USB \cite{Mondal2023MinorityOF} and ASB \cite{Ramaswamy_2021_CVPR}) over three datasets and various bias ratios. Note how our method (FFR) either matches or improves the worst and balanced accuracy of ASB and USB over each dataset and each bias ratio. For example, FFR improves over USB and ASB on UTK-Face and Bias ratio 95\% by over $10\%$. This is because, as we discuss in Section \ref{sec:ffr}, FFR addresses the bias between real and synthetic data and, thus, is more able to use both data sources to mitigate the bias effectively.  

\begin{figure*}[t]
    \centering % <-- added
\begin{subfigure}{0.32\textwidth}
  \includegraphics[width=\linewidth]{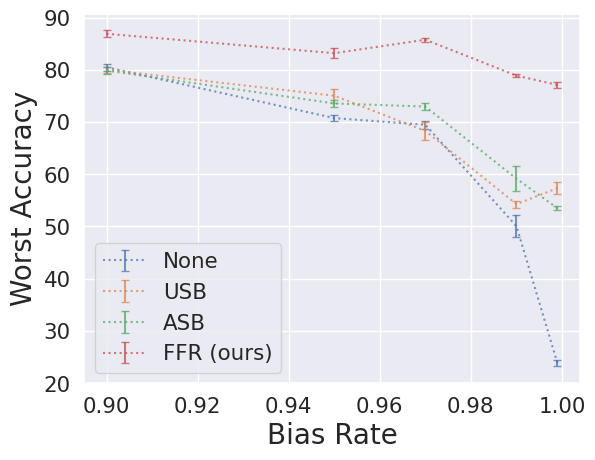}
\end{subfigure}\hfil % <-- added
\begin{subfigure}{0.32\textwidth}
  \includegraphics[width=\linewidth]{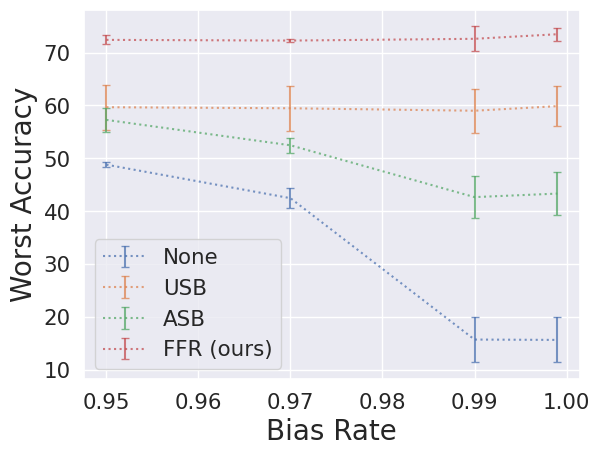}
\end{subfigure}\hfil % <-- added
\begin{subfigure}{0.32\textwidth}
  \includegraphics[width=\linewidth]{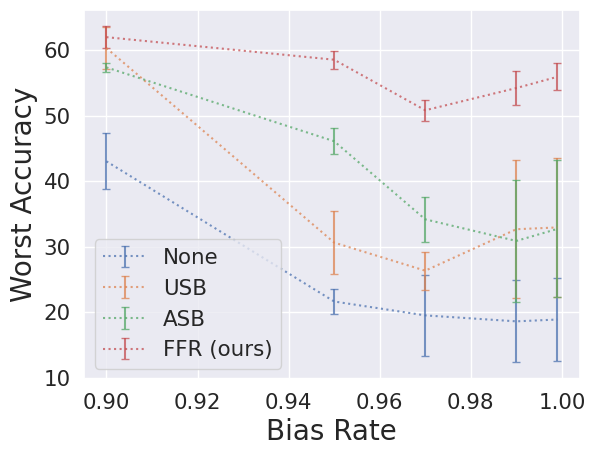}
\end{subfigure}
\medskip

\begin{subfigure}{0.32\textwidth}
  \includegraphics[width=\linewidth]{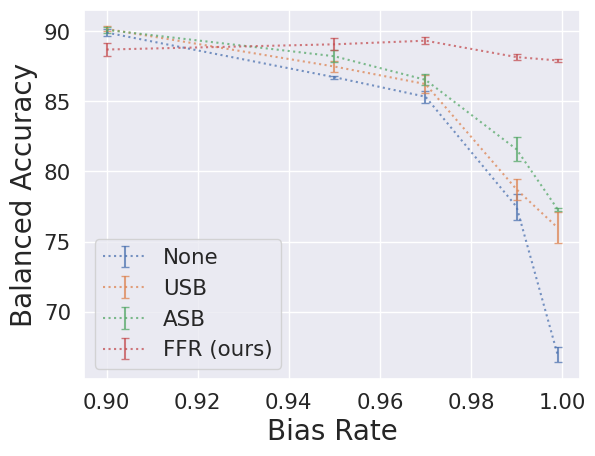}
  \caption{CelebA-HQ \cite{CelebAMask-HQ}}
  \label{figure_3_a}
\end{subfigure}\hfil % <-- added
\begin{subfigure}{0.32\textwidth}
  \includegraphics[width=\linewidth]{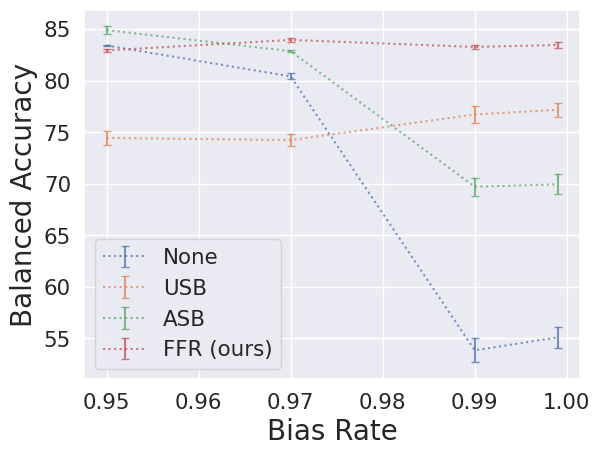}
  \caption{SpuCO Animals \cite{joshi2023towards}}
  \label{figure_3_b}
\end{subfigure}\hfil % <-- added
\begin{subfigure}{0.32\textwidth}
  \includegraphics[width=\linewidth]{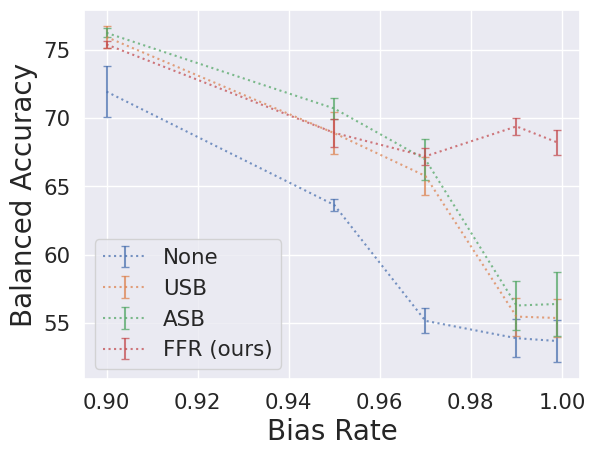}
  \caption{UTK-Face \cite{zhifei2017cvpr}}
  \label{figure_3_c}
\end{subfigure}
\caption{Comparison of performance between the effect: (None) no synthetic data is used, (USB) synthetic data is used to uniformly balance the distribution (extension of prior work on imbalanced classification \cite{Mondal2023MinorityOF}),  (ASB) balanced synthetic data is added to the real dataset \cite{Ramaswamy_2021_CVPR} and (FFR) our method where pretrain on balanced synthetic data and fine tune on real data. Models are trained with ERM. Note how our method either matches or improves the performance of prior work augmentation methods. Refer to Section \ref{sec:systemic_analysis_exp} for discussion.} 
\label{fig:figure_3}
\vspace{-5mm}
\end{figure*}

More notably, we find that the augmentation methods of prior work result in stable performance on SpuCo Animals across bias ratios, but their performance decreases significantly as bias increases on CelebA-HQ and UTK-Face. This is unlike our method, where the performance remains stable. This demonstrates that our method is more robust to more severe bias.

\subsection{Combining Synthetic Augmentation Methods with Synthetic-Data-Free Bias Mitigation Methods. }
\label{sec:harder_analysis_exp}

In this Section, we combine the synthetic data augmentation methods, namely prior work USB \cite{Mondal2023MinorityOF} and ASB \cite{Ramaswamy_2021_CVPR}, and our method FFR with synthetic-data-Free bias mitigation methods: GroupDRO \cite{sagawa2019distributionally}, DFR \cite{kirichenko2022last}, and Resampling. As we noted at the beginning of Section \ref{sec:main_exps} under Baselines, when combining these methods with ASB and USB, we deploy the synthetic-data-free bias mitigation methods to address the bias toward both $(B, G)$. However, when deployed with FFR, they are implemented to address the bias toward $B$ (the bias toward $G$ is automatically addressed by FFR (Section \ref{sec:ffr})). Figure \ref{fig:figure_4} reports the averaged performance over the three datasets and all bias ratios. For simplicity, denote the synthetic-data-free  methods (GroupDRO, Resampling, DFR) as \textbf{SD-Free} and synthetic-data-augmentation methods (USB, ASB, and FFR) as \textbf{SD-Aug}. Below, we consider the impact of \textbf{SD-Free} on \textbf{SD-Aug} (SD-Free $\xrightarrow[]{}$ SD-Aug). Then we consider the impact of  \textbf{SD-Aug} on \textbf{SD-Free} (SD-Aug $\xrightarrow[]{}$ SD-Free).

\begin{figure}[t]
    \centering % <-- added
\begin{subfigure}{0.48\linewidth}
  \includegraphics[width=\linewidth]{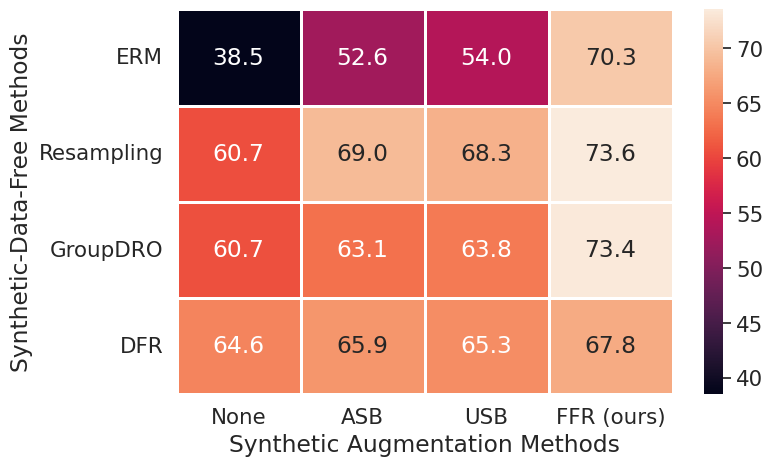}
  \caption{Worst Accuracy}
\end{subfigure}\hfil % <-- added
\begin{subfigure}{0.48\linewidth}
  \includegraphics[width=\linewidth]{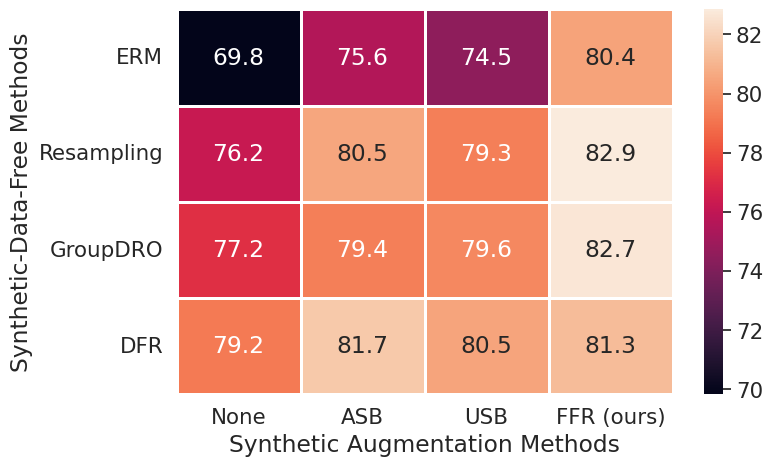}
  \caption{Balanced Accuracy}
\end{subfigure}\hfil % <-- added

\caption{Comparing the performance of synthetic-data-free bias mitigation methods, namely GroupDRO \cite{sagawa2019distributionally}, Resampling, and Deep Feature Reweighting (DFR) \cite{kirichenko2022last} with no synthetic augmentation (None) as well as with synthetic augmentation using prior work methods (USB \cite{Mondal2023MinorityOF} and ASB \cite{Ramaswamy_2021_CVPR}) and our method FFR. Performance is averaged across three datasets and five bias ratios. Note how our method (FFR) in column four is best at improving the performance of non-synthetic-data augmentation methods. Refer to \ref{sec:harder_analysis_exp} for further discussion. } 
\label{fig:figure_4}
\vspace{-5mm}
\end{figure}

\smallskip\noindent\textbf{SD-Free $\xrightarrow[]{}$ SD-Aug}   Note the change in performance from top to bottom in Figure \ref{fig:figure_4}. SD-Free methods significantly improve the performance of SD-Aug methods. For example, the average worst accuracy of ASB improves by 16.4\% (goes from 52.6 \% to 69.0 \%).. We note a similar trend with USB. This is likely because SD-Free methods address some of the bias between data distributions (real and synthetic) that USB and ASB fail to address. However, even when combined with SD-Free methods, ASB and USB still lag behind FFR  even when no SD-Free methods are used (namely row 1 in the Figure where FFR shows 70.3 worst accuracy). This is likely because FFR addresses the bias toward $(B, G)$ by definition while SD-Free methods have to deal with double the number of bias groups in $B$ to address the bias toward $(B, G)$. As a result, this renders the optimization procedure more difficult especially in high bias settings where few samples of the minority groups are available. Overall, these results indicate that FFR by itself is a simple yet effective method of mitigating bias. Nevertheless, when FFR is combined with SD-Free, we see some improvements with GroupDRO and Resampling (by about 3 points on worst accuracy). This is likely the result of addressing some of the bias toward $(B)$ that might be learned in Step 2 of FFR, as discussed in Section \ref{sec:ffr}. 

\smallskip\noindent\textbf{SD-Aug $\xrightarrow[]{}$ SD-Free} Note the change in performance from left to right in Figure \ref{fig:figure_4}. Overall, the performance of SD-Free methods improve as a result of using SD-Aug methods. This is likely the result of SD-Aug methods improving the representations of minority groups, especially in high-bias settings where few samples of the minority groups are available. More notably, the worst accuracy \textit{most} improves when using our SD-Aug method (FFR), where it improves Resampling, GroupDRO, and DFR by $13\%, 13\%, 4\%$, respectively. This is because, as we discussed in the previous section, our method automatically addresses the synthetic-real bias (\ie bias toward $(B, G)$).

\subsection{FFR Design Ablations}
\label{sec:stages_ablation}

FFR is composed of two stages. Stage 1: Pretraining on balanced synthetic data and Stage 2: Fine tuning on real data. Pretraining is done using ERM, and Fine tuning could be done with ERM or bias mitigation methods like Group-DRO and Resampling, which yield further improvements as discussed in Section \ref{sec:ffr}. In this Section, we study the effect of FFR Stage and Pretraining choices.
\smallskip

\smallskip\noindent\textbf{FFR Stages.} Table \ref{tb:ffr_stages} reports the performance of FFR using Stage 1 only, Stage 2 only, Stage 2 followed by Stage 1, and then Stage 1 followed by Stage 2. Note that Stage 2 (table line 2) by itself yields poor performance. This is expected as Stage 2 amounts to training a model with ERM on the biased data without pretraining on synthetic data. Thus, the model, as expected, learns the bias. Training with Stage 1 by itself (table line 1) while improving performance over Stage 2 by itself doesn't match the performance of FFR (Stage 1 $\xrightarrow[]{}$ Stage 2). This is likely due to the real and synthetic data distribution gap. Therefore, following up Stage 2 with Stage 1 is important to bridge the performance gap. Finally, note that reversing FFR (table line 3) doesn't match the performance of FFR. This is likely because the model overfits over the synthetic data. 
\smallskip

\begin{table}[t!]
\sisetup{table-number-alignment=center}
\sisetup{
  table-align-uncertainty=true,
  separate-uncertainty=true,
  detect-all,
  detect-weight=true,
  detect-shape=true, 
  detect-mode=true,
}
\centering 
\caption{Ablation of FFR stages over SpuCO Animals averaged over all bias ratios. Note how the inclusion of both stages in our chosen order ($1 \xrightarrow{} 2$) achieves the best performance, confirming the importance of our method design. Refer to Section \ref{sec:stages_ablation} for further discussion. }
\label{tb:ffr_stages}

    \begin{tabular}{l  S[table-format=2.1 (1),detect-weight=true] S[table-format=2.1(1),detect-weight=true]}
        \toprule
         & WA & BA \\
        \midrule
        Stage 1                                  & 51.2\kern-2.3em $\scriptstyle\pm0.3$             & 77.0 \kern-2.3em $\scriptstyle\pm 0.9$           \\
        Stage 2                                  & 30.6\kern-2.3em $\scriptstyle\pm5.4$             & 68.1\kern-2.3em $\scriptstyle\pm1.2$             \\
        Stage 2 $\xrightarrow{}$ Stage 1           & 59.6 \kern-2.3em $\scriptstyle\pm 5.4$           & 82.8 \kern-2.3em $\scriptstyle \pm 0.9$          \\
        Stage 1 $\xrightarrow{}$ Stage 2: FFR (ours) & \bfseries 72.6 \kern-2.3em $\scriptstyle\pm 2.3$ & \bfseries 83.4 \kern-2.3em $\scriptstyle\pm 0.3$ \\
        \bottomrule
    \end{tabular}
    \vspace{-3mm}
\end{table}

\begin{table}[t!]
\sisetup{table-number-alignment=center}
\sisetup{
  table-align-uncertainty=true,
  separate-uncertainty=true,
  detect-all,
  detect-weight=true,
  detect-shape=true, 
  detect-mode=true,
}
\caption{Ablation of the pretraining distribution used with FFR over CelebA-HQ averaged over all bias ratios. Note how using a balanced distribution during pretraining is important to achieve good performance. Refer to Section \ref{sec:stages_ablation} for further discussion.}
\label{tb:ffr_pretraining}
\centering 
    \begin{tabular}{l  S[table-format=2.1 (1),detect-weight=true] S[table-format=2.1(1),detect-weight=true]}
        \toprule
         & WA & BA \\
        \midrule
        No Synthetic Augmentation & 58.9\kern-2.3em $\scriptstyle\pm1.8$ & 81.2\kern-2.3em $\scriptstyle\pm0.9$ \\
        FFR w/ Biased Pretraining & 69.6\kern-2.3em $\scriptstyle\pm3.6$ & 84.4 \kern-2.3em $\scriptstyle\pm1.0$ \\
        FFR w/ Balanced Pretraining & \bfseries 82.3\kern-2.3em $\scriptstyle\pm1.1$ &  \bfseries 88.6\kern-2.3em $\scriptstyle\pm0.5$ \\
        \bottomrule
    \end{tabular}
\vspace{-5mm}
\end{table}

\smallskip\noindent\textbf{FFR Pretraining.}
Table \ref{tb:ffr_pretraining} compares balanced pretraining to pretraining on a biased synthetic distribution that follows the biased real distribution. Note how the performance drops significantly when a biased distribution is used for pretraining. These results offer compelling evidence that using a balanced synthetic distribution during pretraining is crucial.

\subsection{Empirical Investigation of the Real-Synthetic Data bias}
\label{sec:real-synth-bias-mitigation}

The main motivation behind FFR is mitigating the bias that could arise due to distributional differences between Real and Synthetic data (\eg, Synthetic artifacts \cite{Corvi_2023_ICASSP}) when addressing dataset bias. Assuming these differences, we prove in Section \ref{sec:motivation} that, in the standard
spurious-correlation setting considered in this work, every attempt to balance
a biased dataset with synthetic data results in a new bias toward the pair
$(B,G)$, where $B$ denotes the original dataset bias categories and $G$
denotes whether the image is synthetic or real. FFR addresses this issue by simply dedicating a training step for each data source. The positive impact of FFR is evident from the improved worst accuracy performance noted in Section \ref{sec:systemic_analysis_exp}. In this Section, we seek to further verify this claim through two additional experiments outlined below.

\begin{figure}[t]
 \centering % <-- added
  \includegraphics[width=\linewidth]{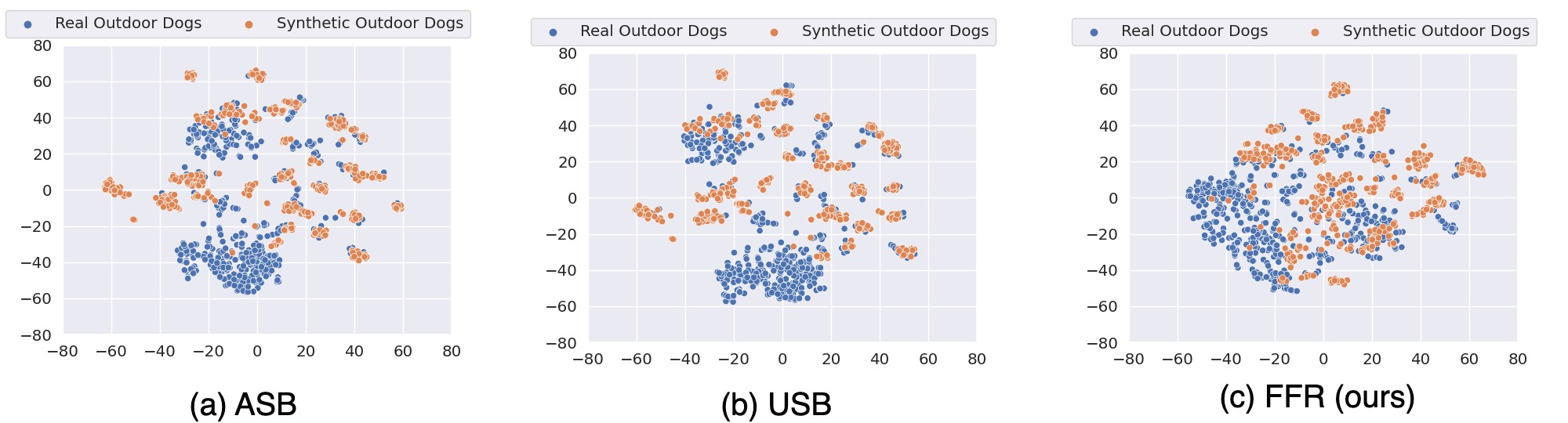}
\caption{Comparing the projections of Real vs. synthetic Data using t-SNE \cite{van2008visualizing} with prior work synthetic augmentation (USB \cite{Mondal2023MinorityOF} and ASB \cite{Ramaswamy_2021_CVPR}) and our synthetic Augmentation method (FFR). Note how our method (FFR) is the best method for projecting Real and Synthetic data close to each other. This is likely because FFR is less impacted by the bias between real and synthetic data and, thus, is posed to learn best from the two data sources. Refer to Section \ref{sec:real-synth-bias-mitigation} for further discussion. } 
\label{fig:figure_5}
\vspace{-4mm}
\end{figure}

 \begin{table}[t!]
\sisetup{table-number-alignment=center}
\sisetup{
  table-align-uncertainty=true,
  separate-uncertainty=true,
  detect-all,
  detect-weight=true,
  detect-shape=true, 
  detect-mode=true,
}
\caption{Comparison between Uniform Synthetic Balancing (USB), Additive Synthetic Balancing (ASB), and From Fake to Real (FFR) on Synthetic Data versus Real Data using the SpuCO Animals Dataset. Results are averaged over all bias ratios. Refer to Section \ref{sec:real-synth-bias-mitigation} for discussion.}
\label{tb:real_synth_bias_exp}
\centering 
    \begin{tabular}{l  S[table-format=2.1 (1),detect-weight=true] S[table-format=2.1(1),detect-weight=true]  S[table-format=2.1 (1),detect-weight=true] S[table-format=2.1(1),detect-weight=true]}
        \toprule
         & \multicolumn{2}{c}{Real Data} & \multicolumn{2}{c}{Synthetic Data} \\
         & WA & BA & WA & BA \\ \hline
        % {ERM} & &   &   &  \\
        {USB} & 59.5 \kern-2.3em $\scriptstyle \pm 8.2$ & 75.6 \kern-2.3em $\scriptstyle\pm 1.3$ &  68.7\kern-2.3em $\scriptstyle\pm0.2$ & 82.5 \kern-2.3em $\scriptstyle\pm0.4$ \\
        {ASB} & 48.9\kern-2.3em $\scriptstyle\pm5.8$ & 76.8\kern-2.3em $\scriptstyle\pm1.1$ & 80.7\kern-2.3em $\scriptstyle \pm 0.3$ & 91.1\kern-2.3em $\scriptstyle\pm0.1$ \\
        {FFR (ours)} & \bfseries 72.6\kern-2.3em $\scriptstyle\pm2.3$ & \bfseries 83.4\kern-2.3em $\scriptstyle\pm0.3$ & \bfseries 89.2\kern-2.3em $\scriptstyle\pm0.2$ & \bfseries 96.3\kern-2.3em $\scriptstyle\pm0.3$\\
        \bottomrule
        \end{tabular}
    \vspace{-4mm}
\end{table}

\smallskip\noindent\textbf{FFR projects Real and Synthetic image embeddings more tightly}. If prior work's synthetic augmentation methods are biased with respect  to $(B, G)$, then they likely use different features per data source when making a prediction (\eg, the model will use the synthetic artifacts when making predictions about the synthetic data). This, in turn, will likely mean that the synthetic data embeddings are clustered separately from those of real data. However, if the method is not impacted by the real-synthetic bias, then that means that it uses the same correct core features per data source (\eg, features about the dog rather than synthetic artifacts when deciding if the dog is a small dog or a big dog). To verify this claim, in Figure \ref{fig:figure_5} we plot the t-SNE \cite{van2008visualizing}  projections of USB \cite{Mondal2023MinorityOF}, ASB \cite{Ramaswamy_2021_CVPR} and FFR real vs.\ synthetic embeddings. Note how both ASB and USB clearly project real and synthetic data into two separate clusters. This indicates that ASB and USB likely use features unique to the real vs.\ synthetic data (\eg, artifacts). However, our method (FFR) projects these samples more tightly indicating that it uses the same unbiased core features.

\smallskip\noindent\textbf{FFR is better at learning from Synthetic Data}. If prior work's synthetic augmentation methods are impacted by the bias toward $(B, G$), then they not only will have learned a biased behavior on the real data, but also on the synthetic data. To verify this, observe the worst accuracy compared to the balanced accuracy on real versus synthetic data in Table \ref{tb:real_synth_bias_exp}. Note how both USB \cite{Mondal2023MinorityOF} and ASB \cite{Ramaswamy_2021_CVPR} perform poorly (low worst accuracy compared to balanced accuracy) on both the real and synthetic data. This indicates that rather than using synthetic data to mitigate the bias and generalize to real data, both methods learned to be biased against minority groups on both the real and synthetic data. However, our method (FFR), doesn't suffer from this issue. Specifically,  FFR worst-group accuracy on both the synthetic and real data is higher than ASB and USB and closer to FFR balanced accuracy indicating significantly less biased behavior.

\begin{figure}[t!]
    \centering
    \includegraphics[width=0.9\linewidth]{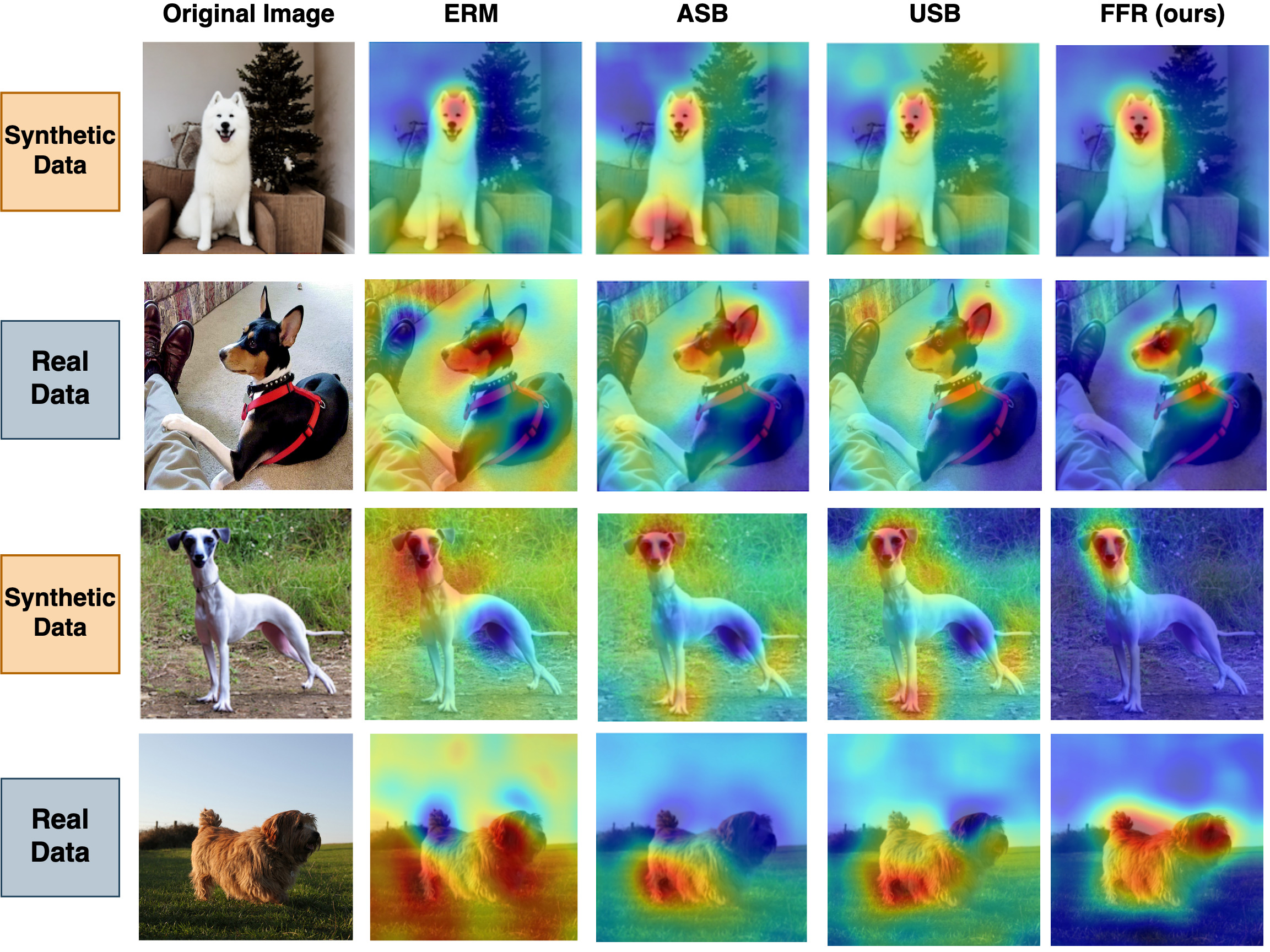}
    \caption{ Saliency maps using RISE \cite{Petsiuk2018rise} when predicting Big Dogs (top two rows) and Small dogs (bottom two rows) using ERM, ASB \cite{Ramaswamy_2021_CVPR}, USB \cite{wang2020deep,Mondal2023MinorityOF} and our method FFR to augment the dataset with synthetic data. The real images are from  SpuCO Animals \cite{joshi2023towards}, and the synthetic data is from Stable Diffusion v1.4 \cite{Rombach_2022_CVPR}. Note how our method (FFR) is the only method that can localize the relevant dog features and not get distracted by spurious background features. Refer to Section \ref{sec:qualt} for discussion. } 
    \label{fig:figure_6}
    \vspace{-6mm}
\end{figure}

\subsection{Qualitative Analysis}
\label{sec:qualt}

In this Section, we conduct a qualitative comparison between ERM without any synthetic data, Additive Synthetic Balancing (ASB) \cite{Ramaswamy_2021_CVPR}, Uniform Synthetic Balancing (USB) \cite{Mondal2023MinorityOF} , and our method From Fake to Real (FFR) on the SpuCo Animals dataset \cite{joshi2023towards} with bias rate $99.9\%$. Note that the dataset contains four classes: Big Dogs, Small Dogs, Landbirds, and Waterbirds. In this Section, we focus on the minority subgroups ``Big Dogs Indoors’’ and ``Small Dogs outdoors’’ and sample a real and synthetic image from each subgroup. For each image and model, we produce a saliency map using RISE \cite{Petsiuk2018rise}. Figure \ref{fig:figure_5} reports our results, where we find FFR is the only method that is able to focus on the dog features while disregarding features from the background in both the synthetic and real images. For example, in the second row, both ASB and USB pay attention to the man's feet as well as the ground floor and what seems to be the bottom of a couch to make predictions. Whereas our method (FFR) only focuses on the dog's features. More interestingly, note how for the synthetic images in rows 1 and 3, prior work methods (ASB and USB) use generative artifacts (\eg, three ``toes’’ for the dog rather than four) to make predictions, whereas our method (FFR) ignores these features. Thus, our method is effective at resolving the issue of bias toward the distributional differences between real and synthetic data.

\section{Conclusion}

We demonstrated through empirical and theoretical work that bias mitigation methods which augment biased datasets with synthetic data fail to address a bias due to distributional difference between real and synthetic data.  To address this issue, we introduced From Fake to Real (FFR): a framework that separates training on synthetic data from training on real data, thus, avoiding the bias between the two data sources. Our systemic analysis over three datasets and five bias settings per dataset demonstrated how our method improved worst group accuracy over prior work methods by up to 20\%. Furthermore, FFR continued to show superior performance even when methods where combined with synthetic-data-free methods. Finally, we provided an extensive ablation that confirms our methods design choices including the pretraining and stage choices. 

\smallskip \noindent\textbf{Limitations and Future Work} In our work, we use large pre-trained text-to-image models to generate synthetic data. While the property of controllable generation using text allows us to generate data that undoes the bias of the real dataset, the generative model might nevertheless inject some biases into the generated data that are not accounted for by the text used to generate the images. For example, Stable Diffusion \cite{Rombach_2022_CVPR} used in this work has been demonstrated to exhibit several biases \cite{luccioni2023stable,bianchi2023easily}. Moreover, as noted in Table \ref{tb:ffr_pretraining}, FFR relies on the generative model being able to faithfully generate a balanced synthetic dataset to achieve good performance; an imbalanced pretraining distribution significantly hurts performance.  However, as prior work noted \cite{liu2023discovering}, recent diffusion models occasionally struggle to follow some prompts espeically ones that require compositionality \cite{gokhale2022benchmarking,marcus2022very}. Therefore, this might jeopardize the ability of diffusion models to generate a balanced pretraining distribution in some cases and thus hurt FFR performance. Therefore, future research that focuses on training fairer and more accurate generative models would alleviate some of these issues. Nevertheless, our approach is generative model agnostic as it is addressing the issue of bias due to data source bias.

\noindent\textbf{Acknowledgements} This material is based upon work supported, in part, by DARPA under agreement number HR00112020054 and the National Science Foundation, including under Grant No.\ 2120322. Any opinions, findings, and conclusions or recommendations expressed in this material are those of the author(s) and do not necessarily reflect the views of the supporting agencies.

% \clearpage  % TODO REVIEW/FINAL: This \clearpage needs to be removed from both review and camera-ready versions.

% ---- Bibliography ----
%
% BibTeX users should specify bibliography style 'splncs04'.
% References will then be sorted and formatted in the correct style.
%
\bibliographystyle{splncs04}
\bibliography{main}

\appendix
\section{Theorem 1 Proof}
\label{apd:proofs}

In our work, we note how prior work synthetic augmentation methods for bias mitigation fail to account for a bias between the real and synthetic distributions. More concretely, assuming real and synthetic data are drawn from different
distributions, we consider the standard spurious-correlation setting defined
in Theorem~1 and argue that it is impossible to guarantee that we can create
$\bar{D}$ where $Y$ is not biased toward the pair $(B,G)$. To prove this, we first prove the following helpful Lemma: 

\begin{lemma}
\label{lemma:1}
 Assume that $P_{D}(Y | B) = P_{D}(Y)$, then for any $y, y^{\prime} \in Y$ and $b \in B$, we get $P_{D}(B = b | Y = y) = P_{D}(B = b | Y = y^{\prime})$
\end{lemma}
\begin{proof}

Given any $b \in B$: 

\begin{align}
    P_{D}(B = b | Y = y) &= \frac{P_{D}(Y = y | B = b) P_{D}(B = b)}{P_{D}(Y = y)} \nonumber \\ 
    &= \frac{P_{D}(Y = y)P_{D}(B)}{P_{D}(Y = y)} \nonumber \\
    &= P_{D}(B = b)
\end{align}

Note that (2) simply follows by definition of  $P_{D}(Y | B) = P_{D}(Y)$. Similarly:

\begin{align}
    P_{D}(B = b | Y = y^{\prime}) &= \frac{P_{D}(Y = y^{\prime} | B = b) P_{D}(B = b)}{P_{D}(Y = y^{\prime})} \nonumber\\
    &= \frac{P_{D}(Y = y^{\prime})P(B)}{P_{D}(Y = y^{\prime})} \nonumber \\ 
    &= P_{D}(B = b)
\end{align}

Thus, 

\begin{align}
     P_{D}(B = b | Y = y) &= P_{D}(B = b) \nonumber \\
     &= P_{D}(B = b | Y = y^{\prime}) 
\end{align}

\end{proof}

\theoremone

\begin{proof}

We will prove this by contradiction. Assume that $P_{\bar{D}}(Y | B, G) = P_{\bar{D}}(Y)$. By the spurious-correlation assumption in Theorem~1, there exist
$b,b' \in B$ and $y,y' \in Y$ such that

\begin{enumerate}
    \item $P_{D}(B = b | Y = y) > P_{D}(B = b^{\prime} | Y = y)$
    \item $P_{D}(B = b^{\prime} | Y = y^{\prime}) > P_{D}(B = b | Y = y^{\prime}) $
\end{enumerate}

Now, assume $Count_{D}(Y=y, B=b, G=g)$ is an operator that returns the number of samples given class $y$, bias $b$ and real/synthetic label $g$ in dataset $D$. Moreover, denote the following variables:

\begin{enumerate}
    \item $M = Count_{\bar{D}}(B = b, Y = y, G = real)$
    \item $N =  Count_{\bar{D}}(B = b^{\prime}, Y = y, G = real)$
    \item $M^{\prime} = Count_{\bar{D}}(B = b, Y = y^{\prime}, G = real)$
    \item $N^{\prime} = Count_{\bar{D}}(B = b^{\prime}, Y = y^{\prime}, G = real)$
\end{enumerate}

Then it follows that given any $\bar{D} \in \bar{\mathcal{D}}$, then:

\begin{enumerate}
    \item $M > N$
    \item $ N^{\prime} > M^{\prime}$
\end{enumerate}

Now, observe: 

\begin{align}
    P_{\bar{D}}(B = b, G = real | Y = y) \nonumber &= \frac{M}{\splitfrac{\sum_{b} Count_{\bar{D}}(B=b, Y=y, G=Real)} {+ \sum_{b} Count_{\bar{D}}(B=b, Y=y, G=Synthetic)}}  \nonumber \\
    &> \frac{N}{\splitfrac{\sum_{b} Count_{\bar{D}}(B=b, Y=y, G=Real)} {+ \sum_{b} Count_{\bar{D}}(B=b, Y=y, G=Synthetic)}} \nonumber \\ 
    &=  P_{\bar{D}}(B = b^{\prime}, G = real | Y = y)
\end{align}

Similarly: 

\begin{align}
    P_{\bar{D}}(B = b^{\prime}, G = real | Y = y^{\prime}) \nonumber &= \frac{N^{\prime}}{\splitfrac{\sum_{b} Count_{\bar{D}}(B=b, Y=y^{\prime}, G=Real)} {+ \sum_{b} Count_{\bar{D}}(B=b, Y=y^{\prime}, G=Synthetic)}} \nonumber\\
    &> \frac{M^{\prime}}{\splitfrac{\sum_{b} Count_{\bar{D}}(B=b, Y=y^{\prime}, G=Real)} {+ \sum_{b} Count_{\bar{D}}(B=b, Y=y^{\prime}, G=Synthetic)}}  \nonumber. \\ 
    &=  P_{\bar{D}}(B = b, G = real | Y = y^{\prime})
\end{align}

In order to satisfy the main assumption in our proof, \ie, $P_{\bar{D}}(Y | B, G) = P_{\bar{D}}(Y)$, then following the contrapositive of  Lemma \ref{lemma:1}: 

\begin{align}
    P_{\bar{D}} & (B = b^{\prime}, G = real | Y = y^{\prime}) =   P_{\bar{D}}(B = b^{\prime}, G = real | Y = y) 
\end{align}

and 

\begin{align}
    P_{\bar{D}}&(B = b, G = real | Y = y^{\prime}) =  P_{\bar{D}}(B = b, G = real | Y = y)
\end{align}

To that end, we can change the term:  $\sum_{b} Count_{\bar{D}}(B=b, Y=y, G=Synthetic)$ by adding more synthetic data to the dataset. We can't change $\sum_{b} Count_{\bar{D}}(B=b, Y=y, G=Real)$ because we don't have access to more real data. Therefore, according to the results in (4) and (5), adding more synthetic data to achieve (6) implies that: 

\begin{align}
   P_{\bar{D}}&(B = b, G = real | Y = y) > \nonumber P_{\bar{D}}(B = b, G = real | Y = y^{\prime})
\end{align}

which breaks (7). Similarly, achieving (7) by adding more synthetic data implies that:

\begin{align}
   P_{\bar{D}}&(B = b^{\prime}, G = real | Y = y^{\prime}) > \nonumber P_{\bar{D}}(B = b^{\prime}, G = real | Y = y)
\end{align}

which breaks (6). Thus, arriving to a contradiction. 
 
\end{proof}

\begin{table}[t!]
\sisetup{table-number-alignment=center}
\sisetup{
  table-align-uncertainty=true,
  separate-uncertainty=true,
  detect-all,
  detect-weight=true,
  detect-shape=true, 
  detect-mode=true,
}
\caption{Comparing the effect of lower quality synthetic data on the performance of USB, ASB, and our method FFR on UTK-Face averaged over five bias ratios. Refer to Supplementary Section \ref{sec:lower_qual_synth_appendix} for discussion.}
\label{tb:corrupt_synth}
\centering 
    \begin{tabular}{l  S[table-format=2.1 (1),detect-weight=true] S[table-format=2.1(1),detect-weight=true] S[table-format=2.1(1),detect-weight=true] S[table-format=2.1(1),detect-weight=true]}
        \toprule
         & \multicolumn{2}{c}{Severity 0} & \multicolumn{2}{c}{Severity 5} \\ \hline
         & WA & BC & WA & BC \\ \hline  
        USB & 36.5 \kern-2.3em $\scriptstyle\pm12.8$            & 64.2\kern-2.3em $\scriptstyle \pm2.5$            & 29.5\kern-2.3em $\scriptstyle\pm6.7$            & 64.2\kern-2.3em $\scriptstyle\pm1.2$             \\
        ASB & 40.2\kern-2.3em $\scriptstyle \pm10.3$            & 65.3\kern-2.3em $\scriptstyle \pm2.6$            & 32.9\kern-2.3em $\scriptstyle \pm6.1$           & 65.8\kern-2.3em $\scriptstyle \pm2.0$            \\
        FFR & \bfseries 56.2 \kern-2.3em $\scriptstyle \pm 3.7$ & \bfseries 69.8 \kern-2.3em $\scriptstyle\pm 1.3$ & \bfseries 50.6\kern-2.3em $\scriptstyle \pm4.0$ & \bfseries 67.6 \kern-2.3em $\scriptstyle\pm 1.8$ \\
        \bottomrule
    \end{tabular}

    % \vspace{-4mm}
\end{table}

\section{Performance with Lower Quality Synthetic Images}
\label{sec:lower_qual_synth_appendix}
In this Section, we analyze our method (FFR) performance as the quality of synthetic data decreases. To corrupt the images, we use the Pixelate, Defcous Blur, and gaussian noise effects combined from \cite{michaelis2019dragon}. Observe results in Table \ref{tb:corrupt_synth} over UTK-Face averaged over 5 bias ratios where Severity 5 represents the most extreme effect available from \cite{michaelis2019dragon}. Note that while prior work synthetic augmentation methods performance USB \cite{Mondal2023MinorityOF} and ASB \cite{Ramaswamy_2021_CVPR} and our method FFR overall performance decrease, our method (FFR) performance remains best.

\section{Pictorial Representation of Synthetic Augmentation Methods}
\label{sec:pictorial_appndx}

In our work, we present a new synthetic augmentation method FFR that outperforms prior work augmentation methods USB \cite{Mondal2023MinorityOF} and ASB \cite{Ramaswamy_2021_CVPR} . This is because our method addresses the distributional difference issue between real and synthetic data. Figure \ref{fig:figure_8} provides a pictorial comparison of the distributions of synthetic and real data between FFR, ASB, and USB. As we discuss in our paper, FFR separates the two data sources into two distributions and thus avoids exposing the model to the statistical differences between the distributions.

\section{Bias Agnostic Synthetic Data Generation}
\label{sec:bias_agnostic_appendix}

The synthetic data used in our method FFR are generated with Stable Diffusion using the following prompt template: \textit{A photo of \{\textbf{bias}\} \{\textbf{class}\}}. In this Section, we investigate the effect of generating bias agnostic synthetic data, \ie, generating data using the template: \textit{A photo of \{\textbf{class}\}}. Table \ref{tb:bias_agnostic} reports the effects of this change on our method using the dataset CelebA-HQ over all bias ratios. Note that performance decreases significantly when using bias-agnostic synthetic data. This is likely because the synthetic data from Stable Diffusion also exhibits gender imbalance. That is why the bias token in the prompt is important to ensure balanced pertaining data and hence the best performance. 

\begin{figure}[t!]
    \centering
    \includegraphics[width=\linewidth]{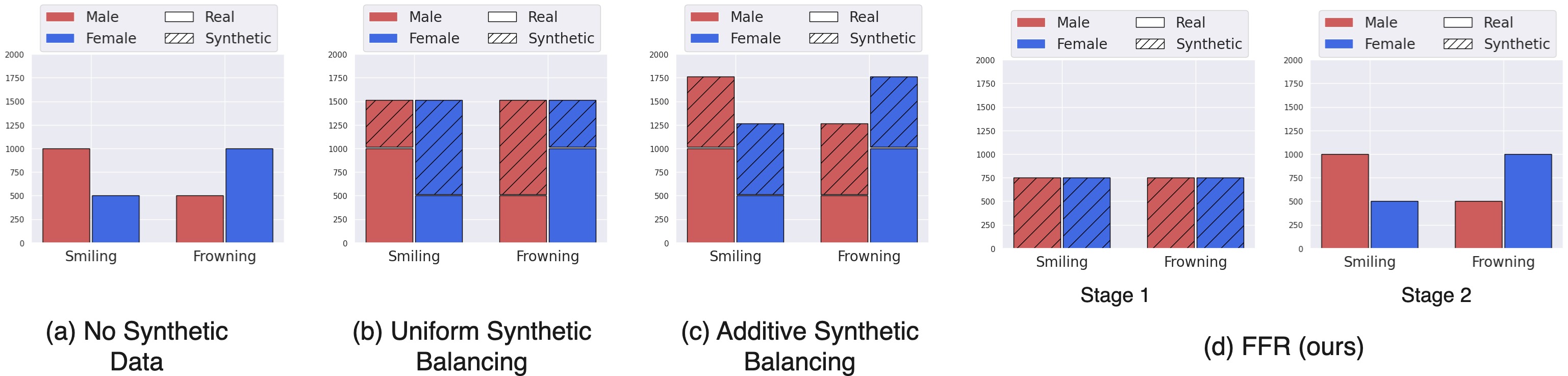}
    \caption{Pictorial representation of prior work synthetic augmentation methods (USB and ASB) and our method FFR. Refer to Supplementary Section \ref{sec:pictorial_appndx} for further discussion.} 
    \label{fig:figure_8}
    \vspace{-1mm}
\end{figure}

\begin{table}[t!]
\sisetup{table-number-alignment=center}
\sisetup{
  table-align-uncertainty=true,
  separate-uncertainty=true,
  detect-all,
  detect-weight=true,
  detect-shape=true, 
  detect-mode=true,
}
\caption{Comparing the effect of bias agnostic synthetic Data versus bias aware data on CelebA-HQ dataset over all bias ratios. Refer to Supplementary Section \ref{sec:bias_agnostic_appendix} for discussion.}
\label{tb:bias_agnostic}
\centering 
    \begin{tabular}{l  S[table-format=2.1 (1),detect-weight=true] S[table-format=2.1(1),detect-weight=true] S[table-format=2.1(1),detect-weight=true]}
        \toprule
         & WA & BC \\ \hline
        ERM & 58.9\kern-2.3em $\scriptstyle\pm1.8$ & 81.2\kern-2.3em $\scriptstyle\pm0.9$\\ \hline
        FFR (Bias Agnostic Data) & 77.6\kern-2.3em $\scriptstyle\pm2.4$ & 86.4\kern-2.3em $\scriptstyle\pm0.8$ \\
        FFR (Bias Aware Data)& \bfseries 82.3\kern-2.3em $\scriptstyle\pm1.1$ & \bfseries 88.6\kern-2.3em $\scriptstyle\pm0.5$ \\
        \bottomrule
    \end{tabular}

    % \vspace{-4mm}
\end{table}

\section{Hyperparameters}
\label{sec:hyper_appendix}

For experiments in Sections 4.1 and 4.2, we provide the learning rates used to train our models for each dataset in Tables \ref{tb:hyperparameters_a} following a grid search over the validation set using the learning rates $\{1e-04, 1e-05, 1e-06, 1e-07\}$. Note that for weight decay, we use $1e-05$ for UTK-Face, CelebA-HQ for SpuCo Animals. Moreover, as discussed in Section 4 in the paper, in Step 2 in FFR, we treat weather the batch norm is frozen or not as a hyperparameter. Table \ref{tb:batch_norm_freeze} shows whether the batch norm was frozen or not for each bias ratio on each dataset. With respect to training length, we train all models for 20 Epochs over all datasets. All models are also pretrained on ImageNet. Finally, when creating the various bias split ratios, we ensure each subgroup has at least 10 samples in the validation set to ensure that the worst group accuracy estimates are not just noise. 

\begin{table}[t!]
\centering
\caption{Learning Rates for datasets UTK-Face and CelebA-HQ and SpuCO Animals in experiments in Sections 4.1 and 4.2 in the main paper. Refer to Supplementary Section \ref{sec:hyper_appendix} for further details.}
\label{tb:hyperparameters_a}

\begin{subtable}{\linewidth}
\centering
\begin{tabular}{l|l|l|l|l|l|l|l|l|l|l|l|l|l|l|l}
\toprule
& & \multicolumn{5}{c|}{CelebA-HQ} & \multicolumn{5}{c|}{UTK-Face} & \multicolumn{4}{c}{Spuco-Animals} \\ \midrule
 & & 90\%  & 95\% & 97\% & 99\% & 99.9\% & 90\%  & 95\% & 97\% & 99\% & 99.9\% & 95\% & 97\% & 99\% & 99.9\% \\  \midrule
\multirow{3}{*}{None} & ERM & 1e-4 & 1e-4 & 1e-4 & 1e-4 & 1e-4 & 1e-4 & 1e-7 & 1e-7 & 1e-4 & 1e-7 & 1e-7 & 1e-7 & 1e-5 & 1e-5 \\ 
 & GroupDRO & 1e-4 & 1e-4 & 1e-4 & 1e-5 & 1e-4 & 1e-4 & 1e-6 & 1e-7 & 1e-4 & 1e-4 & 1e-5 & 1e-7 & 1e-5 & 1e-5 \\ 
 & Resampling & 1e-6 & 1e-6 & 1e-4 & 1e-6 & 1e-6 & 1e-4 & 1e-7 & 1e-7 & 1e-6 & 1e-6 & 1e-7 & 1e-7 & 1e-6 & 1e-6 \\ \midrule
\multirow{3}{*}{USB} & ERM & 1e-4 & 1e-4 & 1e-7 & 1e-4 & 1e-4 & 1e-4 & 1e-7 & 1e-7 & 1e-4 & 1e-4 & 1e-7 & 1e-7 & 1e-7 & 1e-7 \\ 
 & GroupDRO & 1e-4 & 1e-4 & 1e-6 & 1e-4 & 1e-4 & 1e-5 & 1e-7 & 1e-7 & 1e-4 & 1e-4 & 1e-6 & 1e-6 & 1e-5 & 1e-5 \\ 
 & Resampling & 1e-6 & 1e-7 & 1e-7 & 1e-6 & 1e-6 & 1e-4 & 1e-7 & 1e-7 & 1e-6 & 1e-6 & 1e-7 & 1e-7 & 1e-6 & 1e-6 \\ \midrule
\multirow{3}{*}{ASB} & ERM & 1e-4 & 1e-4 & 1e-6 & 1e-4 & 1e-4 & 1e-4 & 1e-7 & 1e-7 & 1e-4 & 1e-4 & 1e-7 & 1e-7 & 1e-5 & 1e-5 \\ 
 & GroupDRO & 1e-4 & 1e-4 & 1e-6 & 1e-4 & 1e-4 & 1e-5 & 1e-7 & 1e-7 & 1e-4 & 1e-5 & 1e-7 & 1e-7 & 1e-5 & 1e-6 \\ 
 & Resampling & 1e-6 & 1e-6 & 1e-4 & 1e-6 & 1e-6 & 1e-4 & 1e-7 & 1e-7 & 1e-6 & 1e-6 & 1e-7 & 1e-7 & 1e-6 & 1e-6 \\ \midrule
\multirow{3}{*}{FFR} & ERM & 1e-7 & 1e-7 & 1e-7 & 1e-7 & 1e-9 & 1e-5 & 1e-8 & 1e-9 & 1e-9 & 1e-9 & 1e-8 & 1e-8 & 1e-8 & 1e-7 \\ 
 & Resampling & 1e-4 & 1e-4 & 1e-7 & 1e-6 & 1e-9 & 1e-5 & 1e-6 & 1e-8 & 1e-6 & 1e-6 & 1e-7 & 1e-7 & 1e-7 & 1e-6 \\ 
 & GroupDRO & 1e-4 & 1e-4 & 1e-7 & 1e-6 & 1e-7 & 1e-5 & 1e-7 & 1e-7 & 1e-5 & 1e-5 & 1e-8 & 1e-8 & 1e-6 & 1e-7 \\ 

\bottomrule
\end{tabular}
\end{subtable}
\end{table}
\begin{table}[t!]
\centering
\caption{Whether batch norm was frozen during Step 2 of FFR on datasets UTK-Face, CelebA-HQ and SpuCO Animals Sections 4.1 and 4.2 experiments in the main paper. T refers to True and F refers to False.  Refer to Supplementary \ref{sec:hyper_appendix} for further details  }
\label{tb:batch_norm_freeze}

\begin{subtable}{\linewidth}
\centering
\begin{tabular}{l|l|l|l|l|l|l|l|l|l|l|l|l|l|l}
\toprule
& \multicolumn{5}{c|}{CelebA-HQ} & \multicolumn{5}{c|}{UTK-Face} & \multicolumn{4}{c}{Spuco-Animals} \\ \midrule
 & 90\%  & 95\% & 97\% & 99\% & 99.9\% & 90\%  & 95\% & 97\% & 99\% & 99.9\% & 95\% & 97\% & 99\% & 99.9\% \\  \midrule
 ERM & T & F & F & F & T & T & T & T & T & T & F & F & F & F \\ 
 GroupDRO & T & T & T & T & F & T & T & T & T & T & F & F & T & F \\ 
 Resampling & T & T & T & T & T & T & F & T & T & T & F & F & T & F \\

\bottomrule
\end{tabular}
\end{subtable}
\end{table}

\begin{table}[t!]
\sisetup{table-number-alignment=center}
\sisetup{
  table-align-uncertainty=true,
  separate-uncertainty=true,
  detect-all,
  detect-weight=true,
  detect-shape=true, 
  detect-mode=true,
}
\caption{Compare the performance between prior work synthetic augmentation methods (USB and ASB) and our method FFR using a ViT-32 backbone. Refer to Supplementary Section \ref{sec:training_vit} for further discussion. }
\label{tb:vit_compare}
\centering 
    \begin{tabular}{l  S[table-format=2.1 (1),detect-weight=true] S[table-format=2.1(1),detect-weight=true]}
        \toprule
         & WA & BC  \\ \hline  
        None & 54.2 \kern-2.3em $\scriptstyle\pm4.0$             & 78.9 \kern-2.3em $\scriptstyle\pm1.2$            \\
        USB  & 64.9 \kern-2.3em $\scriptstyle\pm2.6$             & 82.1\kern-2.3em $\scriptstyle \pm1.4$            \\
        ASB  & 60.1\kern-2.3em $\scriptstyle \pm1.7$             & 81.1\kern-2.3em $\scriptstyle \pm0.7$            \\
        FFR  & \bfseries 82.8 \kern-2.3em $\scriptstyle \pm 1.5$ & \bfseries 87.7 \kern-2.3em $\scriptstyle\pm 0.5$ \\
        \bottomrule
    \end{tabular}

    % \vspace{-4mm}
\end{table}

\section{Training with ViT(s)}
\label{sec:training_vit}

In our work,  we present a new synthetic augmentation method FFR that outperforms prior work augmentation methods USB \cite{Mondal2023MinorityOF} and ASB \cite{Ramaswamy_2021_CVPR}. We use a ResNet50 through all our experiments. In this Section, we test using a ViT 32B \cite{dosovitskiy2020image}. Table \ref{tb:vit_compare} shows the results on the CelebA-HQ dataset. Note how FFR results in the best performance. This is because, as discussed in the main paper, FFR automatically mitigates the bias between the two data sources: real and synthetic data.

\end{document}